\documentclass[10pt]{article}

\usepackage[natbibapa]{apacite} 

\usepackage[utf8]{inputenc}
\usepackage{enumerate}
\usepackage{nameref,hyperref}

\usepackage{changepage}

\usepackage[justification=justified,singlelinecheck=false]{caption}

\usepackage{color}

\definecolor{Gray}{gray}{.25}

\usepackage{graphicx}
\usepackage{epstopdf}

\usepackage{amsmath,amssymb,amsthm, amsfonts,euscript,mathrsfs, mathtools}
\usepackage{bbm}

\usepackage{epsfig, verbatim,enumerate, wrapfig}
\usepackage{graphicx}
\usepackage{epstopdf}
\usepackage[english]{babel}
\usepackage{float}
\usepackage{multirow}
\floatstyle{plaintop}
\restylefloat{table}

\newtheorem*{theorem*}{Theorem}
\newtheorem{theorem}{Theorem}
\newtheorem{lemma}{Lemma}

\newcommand{\wh}{\widehat}
\newcommand{\cond}{{\, \vert \,}}

\makeatletter
\def\verbatim@font{\linespread{1}\normalfont\ttfamily}
\makeatother

\usepackage{lipsum}

\newcommand\blfootnote[1]{%
  \begingroup
  \renewcommand\thefootnote{}\footnote{#1}%
  \addtocounter{footnote}{-1}%
  \endgroup
}

\begin{document}

	% title goes here:

		{\Large
			\textbf\newline{Classification using Ensemble Learning under Weighted Misclassification Loss}
		}
		\newline
		% authors go here:
		\\
		\author{Yizhen Xu$^{1}$, Tao Liu$^{1}$, Michael J. Daniels$^{2}$, Rami Kantor$^{3}$, Ann Mwangi$^{4,5}$, Joseph W Hogan$^{1,4}$}
		Yizhen Xu\textsuperscript{1,*},
		Tao Liu\textsuperscript{1},
		Michael J. Daniels\textsuperscript{2},
		Rami Kantor\textsuperscript{3},
		Ann Mwangi\textsuperscript{4,5},
		Joseph W Hogan\textsuperscript{1,4}
		\\
	
\begin{enumerate}
\item Department of Biostatistics, Brown University,121 S. Main Street, Providence, RI, U.S.A.
\item Department of Statistics and Data Sciences, University of Texas at Austin, Austin TX, U.S.A
\item Division of Infectious Diseases, Brown University, Providence RI, U.S.A
\item Academic Model Providing Access to Healthcare (AMPATH), Eldoret, Kenya
\item College of Health Sciences, School of Medicine, Moi University, Eldoret, Kenya

\end{enumerate}		 

		* \verb|yizhen_xu@brown.edu|\\
		
		* \verb|This is part of thesis work of the first author.|

\blfootnote{This work was supported by NIH (Grant NO: R01 AI108441, P30 AI042853, K24 AI134359, R01 AI066922.}
	
	\section*{Abstract}
	Binary classification rules based on covariates typically depend on simple loss functions such as zero-one misclassification. Some cases may require more complex loss functions. For example, individual-level monitoring of HIV-infected individuals on antiretroviral therapy (ART) requires periodic assessment of treatment failure, defined as having a viral load (VL) value above a certain threshold. In some resource limited settings, VL tests may be limited by cost or technology, and diagnoses are based on other clinical markers. Depending on scenario, higher premium may be placed on avoiding false-positives which brings greater cost and reduced treatment options. Here, the optimal rule is determined by minimizing a weighted misclassification loss/risk.
	
	We propose a method for finding and cross-validating optimal binary classification rules under weighted misclassification loss. We focus on rules comprising a prediction score and an associated threshold, where the score is derived using an ensemble learner. Simulations and examples show that our method, which derives the score and threshold jointly, more accurately estimates overall risk and has better operating characteristics compared with methods that derive the score first and the cutoff conditionally on the score especially for finite samples. 
	
	% now start line numbers

	\newpage
	% the * after section prevents numbering
	
	\section{Introduction}
	Development of accurate binary classification rules is important in many
	areas of biomedical research and clinical practice.  The problem can be
	described as follows:  suppose each individual in a population of interest
	has a binary outcome $Y$ and $p \times 1$ vector of individual-level features $X 
	= (X_1,\ldots,X_p)$.  We would like to find a classification rule 
	$Q: \mathscr{X} \rightarrow \{0,1 \}$ that takes $X \in \mathscr{X}$ as an input
	and generates a classification action $a\in \{0,1\}$ as output.  
	The criterion for classification accuracy is given in terms of a loss function $L(Y,a)$;
	for example, $L(Y,a) = \mathbbm{1}\{a \neq Y\}$ is simple misclassification loss.
	For a given loss criterion, the optimal rule is the one that minimizes 
	expected loss over the population of interest.
	
	Loosely speaking,
	methods for binary and categorical classification can be divided into
	those that perform classification directly and those that
	use a score-and-threshold approach.  Direct classification methods include
	$K$-nearest-neighbors (KNN) \citep{duda1973pattern,cover1967nearest} and tree-based techniques \citep{breiman1984}.
	In score-based methods, classification is carried out by 
	comparing a scalar, real-valued risk score $\Psi(X)$ to a threshold $c$,
	yielding rules of the form $Q(X; \Psi(\cdot), c) = \mathbbm{1}\{\Psi(X) \geq c \}$. 
	The risk score $\Psi(X)$, a function mapping $\mathscr{X}$ to $\rm I\!R$, can represent class membership probability $P(Y=1 \cond X)$ or can be a more general ranking measure, i.e. where $\Psi(X_1) > \Psi(X_2)$ implies that $Y_1$ is stochastically greater than $Y_2$. Moreover, the risk score can be
	derived using a single model, such as a regression model, or using an ensemble method 
	that combines scores from multiple models
	\citep{van2007super}. 
	
	Score-based methods are often used in practice, especially in medical applications. Examples include the VACS score \citep{tate2013internationally} for risk prediction of all-cause mortality among HIV infected individuals and the Nottingham prognostic index \citep{haybittle1982prognostic} for breast cancer. The main motivation here is to build a risk score for classification prediction of virological failure, which can be applied to various stages of individual-level HIV monitoring such as viral load (VL) pooled testing. This is important because VL testing is limited in some resource-limited settings (RLS) and accurate classification based on common clinical markers can reduce cost and mortality. 
	
	A classification rule is optimal if it minimizes average loss over a population of interest.
	The process of deriving optimal classification rules and estimating
	their operating characteristics relies on a training step and a testing step. In the training step, optimal classification rules are estimated; in the testing step, out-of-sample error rates associated with the rules are calculated.  Finding optimal score-based classification rules requires estimation of both
	the score function and the associated threshold.  In this paper, we focus on score-based 
	classification methods under weighted misclassification loss, where the score $\Psi(X)$
	is derived using the Super Learner ensemble \citep{van2007super}. In Super Learner (SL), a library of $K$ learners generates risk scores
	$\Psi_1(X), \ldots, \Psi_K(X)$, which are then combined into a single score using
	the convex combination $\Psi(X; \alpha) = \sum_{k=1}^K \alpha_k \Psi_k(X)$ , where 
	$\alpha = (\alpha_1,\ldots,\alpha_K)$, $\alpha_k\geq 0$ for each $k$,
	and $\sum_k \alpha_k = 1$.  The $\alpha_k$'s are estimated based on cross-validated library learner predictions,
	typically using squared-error, likelihood-based or ROC-based loss. van der Laan et al. \citep{van2007super} showed through applications that SL outperforms individual candidate learners in a given library.

	%A \textit{conditional thresholding} approach is
	%to first use SL to derive a risk score $\wh{\Psi}(X) = \Psi(X;\wh{\alpha})$ and, conditionally %on $\wh{\Psi}(X)$, find $\wh{c}$ that minimizes the loss function
	%for misclassification.  While this may seem intuitive, we show using simulation
	%studies and empirical examples that this approach may lead to over-estimation
	%of misclassification error and may yield classification rules that are not optimal.
	
	A \textit{conditional thresholding} approach is
	to first derive a risk score $\wh{\Psi}(X) = \Psi(X;\wh{\alpha})$ and, conditionally on $\wh{\Psi}(X)$, find $\wh{c}$ that minimizes the loss function
	for misclassification.  In this paper, we regard derivation of optimal decision threshold by directly conditioning on estimated risk of the same training data as conditional thresholding. Because it is intuitively straightforward, conditional thresholding has been used in various studies. Birkner et al. \cite{birkner_creating_2007} used conditional thresholding in predicting 30-day mortality following stroke in a rural Indian population; from a set of candidate learners, they selected the optimal model as the estimator with the smallest cross-validated risk, and chose the cut-off value from the estimated score on the training data to be the smallest value achieving at least $90\%$ sensitivity. Kruppa et al. \cite{kruppa_risk_2012} derived classification rule using conditional thresholding for a GWA study on rheumatoid arthritis; several risk scores were estimated for the training data and the thresholds for binary classification were selected by maximizing Youden index for each risk score in the training data. Shim et al. \cite{shim_development_2018} developped a risk score model to identify patients at high risk of negative effects after total knee arthroplasty, derived the optimal threshold for binary classification directly from the estimated score based on Youden’s index, which is also a conditional thresholding derivation. We show using simulation
	studies and empirical examples that this approach may lead to over-estimation
	of misclassification error and yield classification rules that are sub-optimal. 
	
	As an alternative we propose \textit{joint thresholding}, which combines the estimation of coefficients $\alpha$ and threshold $c$ in one step to gain a better overall performance in terms of out-of-sample prediction risk. It uses information in out-of-sample classification through one cross-validation procedure, and naturally integrates threshold estimation into the SL framework. Simulation and application results under weighted misclassification loss show that our simultaneous estimation approach generally has lower out-of-sample risk compared to conditional thresholding. In this data applications, the threshold estimated using our method is approximately the optimal Bayes threshold when the ensemble score is close to the true class probability, as would be expected under weighted misclassification loss \citep{buja2005loss}.
	
	Our approach involves minimization of an empirical weighted misclassification risk function. The minimization can be difficult because the weighted misclassification risk is composed of step functions that count the misclassification, making it into a nonconvex, nonsmooth and NP-hard optimization problem \citep{chen1996hybrid}.  Conditional on a risk score, estimating $c$ by minimizing the misclassification risk is relatively easy and can be accomplished by line search. In Section \ref{Proposal}, we demonstrate a simultaneous estimation of the ensemble weights and threshold using bounded controlled random search \citep{kaelo2006some}. This method gives us an approximation to the minimizer of the empirical weighted misclassification risk, and is easy to implement.
	
	The paper is structured as follows. Section 2 provides motivating examples and defines weighted misclassification loss. Section 3 describes ensemble learning and the conditional thresholding approach. Section 4 explains our proposal of joint thresholding, discusses related challenges and existing methods for minimization of the weighted misclassification loss/risk, and introduces an approach to joint thresholding by controlled random search. Simulations and data applications are described in Section 5 and 6 respectively. Section 7 gives conclusions and suggestions for further work.

	\section{Weighted Misclassification Loss -- Motivating examples}\label{LOSS}
	
	Total misclassification loss is often the primary criterion for classification, where penalties for false positives and false negatives are the same. However, there are many circumstances where false positive and false negative classifications have different consequences, and should be weighted differently. An instance of this is individual-level HIV treatment monitoring in RLS. Failure of antiretroviral treatment (ART) is indicated by a VL value above a certain threshold. In RLS, VL assessment may be limited by logistics, cost and technology \citep{wang2010advances}; hence, other clinical markers are used at times to predict viral failure. In this situation, false positive classification of VL failure leads to early switching to a second- or third-line ART, which may have higher toxicity and lower adherence, incurs significantly greater cost, and limits treatment options over the long term. False negative classification of viral failure results in lack of regimen switch, risk of drug resistance accumulation and increased morbidity and mortality. These considerations may motivate prioritization of avoiding either false negative or false positive classifications. For example, if avoiding unnecessary switches to second-line therapy is prioritized, the loss function would assign greater weight to false positive misclassifications.
	
	Another example is initial diagnosis of breast cancer malignancy from digitized images. Breast cancer is one of the largest causes of cancer deaths among women \citep{siegel2012cancer}. Patient survival depends largely on early detection and accurate diagnosis. Fine needle aspiration (FNA) \citep{wolberg1994machine} is a minimally invasive procedure that allows measurement of individual cellullar characteristics, allowing algorithm-based cell image analysis for diagnosis. FNA is usually carried out when there is a breast lump previously detected by self-examination or mammography. Positive findings may lead to confirmation of breast cancer through surgical biopsy, which is accurate but requires significantly more recovery time and expense, involves pain and carries risks associated with surgical procedures such as scarring and infection. It may be of interest to place a premium on avoiding false positives over false negatives depending on available tools and expertise, for example, when pathology requirements are not met or availability of anesthesia is limited in RLS \citep{shyyan2006breast,el2011breast}.
	
	%el2011breast describes anesthesia services are available in LRCs, barriers to care posted by pathology and surgery service capacity.
	%shyyan2006breast lists resources needed for surgical biopsy, and related risks/challenges
	
	Our goal is to develop classification rules of the form $Q(X) = Q(X;\Psi(\cdot),c) = \mathbbm{1}\{\Psi(X)\ge c\}$ based on a weighted misclassification loss
	\begin{align*}\label{loss}
	L_\lambda(Y, Q(X)) &= \lambda \mathbbm{1}\{Q(X) = 0, Y = 1\} + (1-\lambda) \mathbbm{1}\{Q(X) = 1, Y = 0\}\\
	&=\lambda \mathbbm{1}\{\Psi(X) < c, Y = 1\} + (1-\lambda) \mathbbm{1}\{\Psi(X)\ge c, Y = 0\},
	\end{align*}
	where $\lambda\in (0, 1)$ is a user-defined weight. The loss function gives weight $\lambda$ for false negative classification (misclassifying $Y=1$ as 0) and $1-\lambda$ for false positive classification (misclassifying $Y=0$ as 1). At $\lambda=.5$ we have total misclassification loss. Setting a value for $\lambda$ depends on specific applications and goals of classification. 
	%Total misclassification loss corresponds to $\lambda = 0.5$.  
	
	Weighted misclassification risk is the expected loss over the underlying joint distribution of $X$ and $Y$,
	\begin{equation} \label{risk}
	\begin{split} 
	R_\lambda(\Psi, c) &= E_{X,Y}\{L_\lambda(Y, Q(X))\}\\
	& = \lambda p P\{\Psi(X) < c \mid  Y = 1\} + (1-\lambda)(1-p) P\{\Psi(X) \ge c\mid  Y = 0\}\\
	& = \lambda p \text{FNR}(\Psi, c) + (1-\lambda) (1-p) \text{FPR}(\Psi, c)
	\end{split}
	\end{equation} 
	where $p = P(Y = 1)$ is the prevalence, $\text{FNR}(\Psi, c) = P\{\Psi(X) < c\mid  Y = 1\}$ the false negative rate, and $\text{FPR}(\Psi, c) =  P\{\Psi(X) \ge c\mid  Y = 0\}$ the false positive rate. 
	
	The objective of the inference problem is to find $\Psi(\cdot)$ and $c$ that minimize the risk. Given a sample of data $(X_1, Y_1), \ldots, (X_n, Y_n)$, this can be operationalized as finding $\Psi(\cdot)$ and $c$ that minimize the empirical risk function,
	\begin{equation}\label{emprisk}
	\widehat{R}_\lambda(Y,X;\Psi, c) = \frac{1}{n}\sum_{i=1}^n \lambda \mathbbm{1}\{\Psi(X_i) < c, Y_i=1\} + (1-\lambda) \mathbbm{1}\{\Psi(X_i) 
	\ge c, Y_i =0\}.
	\end{equation}

	\section{Super Learner with Conditional Thresholding for Classification}\label{Common}
	
	Super learner is an ensemble learner that combines predictions produced by candidate learners from a user defined library $\mathscr{L}$ of $K$ learners. Our definition of ensemble refers to those that combine different base algorithms; for example, we consider bagging and boosting as single learners, as they generate a single strong learner by using one base algorithm to obtain a collection of weak learners, by boostrap aggregation and re-weighting samples respectively. As a kind of stacked generalization \citep{wolpert1992stacked}, SL combines different learning algorithms over the same data, and algorithms such as bagging, boosting, and random forest all can be included as candidate learners of a SL library. 
	
	Consider a library  $\mathscr{L}=(\Psi_1, \ldots, \Psi_K)$ where each candidate learner $\Psi_k\in\mathscr{L}$ is a mapping from $\mathscr{X}$ to $\rm I\!R$, and the prediction score $\Psi(X) = \sum^K_{k=1} \alpha_k\Psi_k(X)$ is a convex combination of predictions from the library of learners.
	Deriving an optimal classification rule requires finding the value of $(\alpha, c)$ that minimizes \eqref{emprisk}.
	
	An intuitive approach to determining the rule is conditional thresholding; that is, first derive a prediction score $\wh{\Psi}(X)=\sum^K_{k=1} \wh{\alpha}_k\wh{\Psi}_k(X)$ using SL, and then identify an optimal threshold value $\wh{c}$ based on the score $\wh{\Psi}(X)$. In SL, the $\alpha$ coefficients are derived from cross validation against loss function such as squared error loss or ROC-based loss. Conditional thresholding approach estimates threshold $c$ by plugging $\wh{\Psi}(X)$ into \eqref{emprisk} and minimizing over $c$. However, the procedure does not reflect out-of-sample performance because it does not use cross-validated weighted misclassification risk.
	
	Conditional thresholding for classification based on SL proceeds as follows:
	\begin{enumerate}
		\item Fit each learner $\Psi_k \in \mathscr{L}$ to the entire data set $\{(Y_i,X_i)\}_{i=1}^n$ and generate score predictions $\wh{\Psi}_k(X_i)$ for $k=1,\ldots,K$ and $i=1,\ldots,n$.
		\item Carry out $D$-fold cross validation. Have $D$ partitions of the data, where partitions are indexed by $d = 1,\ldots,D$. For the $d$th partition, $T(d)$ and $V(d)$ are the training and validation data splits respectively. Fit each candidate learner $\Psi_k$ to $T(d)$, yielding $\wh{\Psi}_{k, T(d)}$. Then generate its prediction on $V(d)$, written as $\wh{\Psi}_{k, T(d)} (X_{V(d)})$, $k = 1,\ldots,K$ and $d = 1,\ldots, D$.
		\item For each candidate learner $\Psi_k$, stack together the $D$ fold-specific predictions $\wh{\Psi}_{k, T(d)} (X_{V(d)})$, $d=1,\ldots,D$ to get $Z_k = \{\wh{\Psi}_{k, T(d)} (X_{V(d)}), d = 1,\ldots,D\}$, an $n\times 1$ cross-validated prediction vector. For the $i$th observation, define $d_i=d(X_i)$ as its validation fold index, i.e. $d_i = 2$ if $X_i\in X_{V(2)}. $Write the $n\times K$ cross-validated prediction matrix as $Z = \{Z_1,\ldots,Z_K\}$, where the $i$th row $k$th column element $Z_{ik}=\wh{\Psi}_{k, T(d_i)} (X_i)$ is an out-of-sample prediction made by the $k$th candidate learner to the $i$th observation, which is a member of validation fold $V(d_i)$. Estimate $\alpha$ in $m(Z;\alpha) = \sum^K_{k=1} \alpha_k Z_k$ by minimizing risk $E\{\tilde{L}(Y,m(Z;\alpha))\}$,
		e.g. ordinary linear regression has quadratic loss and $E(Y\mid Z;\alpha) = m(Z;\alpha)$.
		
		\item Combine $\wh{\alpha}$ from Step 3 with the data predictions $\wh{\Psi}_k(X_i), k = 1,\ldots,K$ from Step 1, and obtain the SL score $\wh{\Psi}_{SL}(X_i;\wh{\alpha}) = \sum^K_{k=1}\wh{\alpha}_k\wh{\Psi}_k(X_i).$
		\item Estimate the classification threshold $c$ by minimizing the empirical risk function \eqref{emprisk} using the SL score as the risk score,
		\begin{align*} 
		\wh{c} &= \operatorname*{argmin}_c \quad\sum^n_{i=1} \lambda \mathbbm{1}\{ \wh{\Psi}_{SL}(X_i;\wh{\alpha}) < c, Y_i=1\} + (1-\lambda)\mathbbm{1}\{\wh{\Psi}_{SL}(X_i;\wh{\alpha}) \ge c,Y_i=0\}
		\end{align*}
		\item For any $x\in \mathscr{X}$ , the classification rule is $\wh{Q}(x) = Q(x; \wh{\Psi}_{SL}(\cdot;\wh{\alpha}),\wh{c}) =  \mathbbm{1}\{\wh{\Psi}_{SL}(x;\wh{\alpha}) \ge \wh{c}\}$.
	\end{enumerate}
	
	\section{Super Learner with Joint Thresholding for Classification} \label{Proposal}
	
	As we will show in empirical examples and simulations, conditional thresholding may over-estimate actual risk. We propose to estimate the classification threshold $c$ based on the cross validated prediction $Z$ (defined in Step 3 of Section \ref{Common}) within the SL algorithm. In our approach, Steps 1 and 2 are the same as above. Steps 3 and 5 are replaced by simultaneous estimation of $\alpha$ and $c$ to satisfy 
	\begin{equation}\label{step3}
	(\tilde{\alpha}, \tilde{c}) = \operatorname*{argmin}_{(\alpha, c)} \sum^n_{i=1}\lambda \mathbbm{1}\{m(Z_i;\alpha)< c, Y_i=1\} + (1-\lambda)\mathbbm{1}\{m(Z_i;\alpha) \ge c, Y_i=0\}.
	\end{equation}
	The SL score in Step 4 and the classification rule in Step 6 are then updated to $\wh{\Psi}_{SL}(X_i;\tilde{\alpha}) = \sum^K_{k=1}\tilde{\alpha}_k\wh{\Psi}_k(X_i)$ and $\wh{Q}(x) = Q(x; \wh{\Psi}_{SL}(\cdot;\tilde{\alpha}),\tilde{c}) =  \mathbbm{1}\{\wh{\Psi}_{SL}(x;\tilde{\alpha}) \ge \tilde{c}\}$ accordingly.
	
	Optimizing the empirical risk in equation \eqref{step3} is complicated by discontinuities introduced by the indicator functions. Common optimization methods such as Newton-Raphson cannot be applied because they require existence of the first or second order derivatives. The lack of smoothness and convexity makes other optimization methods difficult as well. Moreover, minima of the objective function is not unique due to the non-convexity of weighted misclassification loss. 
	
	According to the Bayes rule, the optimal threshold for \eqref{step3} is $1-\lambda$ when $m(Z;\alpha) = P(Y=1\mid X)$. In practice, the threshold $c=1-\lambda$ is valid only when the risk score is a consistent estimate of $P(Y=1\mid X)$ and the sample size is sufficiently large. However, the underlying true mechanism for data generation is very complicated in most applications. Furthermore, whether or not the ensemble score is a probability estimate depends on the investigators' intention and study goal, and a good risk score for classification does not have to be a probability estimate (e.g. SVM).  
	
	One approach to minimizing the empirical weighted misclassification risk is to approximate the weighted misclassification loss with some smooth solvable loss functions. Buja et al. \cite{buja2005loss} used integrals of beta distribution functions to approximate the indicator functions in the loss, theoretically enabling use of optimization algorithms. However, in practice the nonconvexity of the smooth approximation can undermine the invertibility of the Hessian in Newton updates and cause the optimization procedure to fail. 
	
	Another approach is to reformulate the minimization problem as a linear program with equilibrium constraints (LPEC), a special case of a hierarchical mathematical programming that consists of two levels of optimization. Mangasarian \cite{mangasarian1994misclassification} studied total misclassification loss and suggested a Frank-Wolfe type iterative algorithm to approximate the minima that moves the solution towards the minimum of a linear approximation of the objective function in the same domain. Chen and Mangasarian \cite{chen1996hybrid} proposed a hybrid algorithm as an accelerated approximation to the algorithm in Mangasarian \cite{mangasarian1994misclassification}, which is costly in computation. The hybrid algorithm iteratively estimates $\alpha$ by replacing indicator function $\mathbbm{1}(x>0)$ with the convex surrogate $\text{max}(1+x,0)$ at a fixed $c$, and estimates $c$ by minimizing the objective function at a fixed~$\alpha$. Solving LPEC is costly and computationally intensive because the minimization problem is NP-hard \citep{chen1996hybrid}.

	We consider two options to solving equation \eqref{step3} within SL. One is to approximate the solution in two separate steps: (1) using non-negative least squares linear regression to estimate $\tilde{\alpha}$ and then normalizing it to sum to one, and (2) conducting a line search to estimate $\tilde{c}$ conditional on the estimated $\tilde{\alpha}$. We refer to this procedure as \textbf{Two-Step Minimization} in our simulations and data applications. This can be further extended to using a convex and continuous surrogate $\tilde{L}(Y,m(Z;\alpha))$ for estimation of $\tilde{\alpha}$. The process can be described as follows: 
	\begin{enumerate}
		\item[(\ref{step3}a)] Estimate $\tilde{\alpha}$ in  $m(Z;\tilde{\alpha}) = \sum^K_{k=1} \tilde{\alpha}_k Z_k$ by $\operatorname*{argmin}_\alpha E\{\tilde{L}(Y,m(Z;\alpha))\}$, e.g. if $\tilde{L}$ is squared error loss, we use ordinary least squares regression of $Y$ on $Z$.   
		\item[(\ref{step3}b)] Estimate $\tilde{c}$ by conditional minimization using the cross-validated predcitions $Z$, 
		\begin{align*}
		\tilde{c} &= \operatorname*{argmin}_{c} \quad\sum^n_{i=1} \lambda \mathbbm{1}\{m(Z_i;\tilde{\alpha}) < c, Y_i=1\} + (1-\lambda)\mathbbm{1}\{m(Z_i;\tilde{\alpha}) \ge c, Y_i=0\}.
		\end{align*} 
	\end{enumerate}
	
	When the surrogate loss $\tilde{L}$ in Step \ref{step3}a involves threshold $c$, an iterative procedure similar to Chen and Mangasarian \cite{chen1996hybrid} can be used for the estimation of $(\tilde{\alpha}, \tilde{c})$. This two-step procedure provides flexibility in that the user can choose the surrogate loss $\tilde{L}$ based on context. If minimizing the surrogate loss produces risk score that gives good discrimination to the data, the resulting classification rule would be a good approximation to a minimizer of weighted misclassification risk. 
	
	The second option is to estimate $\alpha$ and $c$ using a bounded region optimization, and perform a controlled random search in the bounded region. First, note that the inequality $\sum^K_{k=1} \Psi_k(\cdot)\alpha_k > c$ can be written as $\sum^K_{k=1} \Psi_k(\cdot)\alpha^*_k > c^*$, with $\alpha^*_k = \alpha_k/\alpha_1$ and $c^* = c/\alpha_1$ when $\alpha_1>0$. In our analyis, coefficients $(\alpha_1,\ldots,\alpha_K)$ are constrained to be nonnegative, normalized to sum to one, and without loss of generality $\alpha_1$ is designated as the coefficient to have the largest estimated value from initialization, as described in Step 1 below. 
	
	Assume that the coefficient estimates from initialization based on some convex loss functions are close to the solutions that minimize the weighted misclassification risk. Then, with initialized value of $\alpha^*$ in $[0, 1]^K$, we can estimate $\alpha^*$ and $c^*$ by searching for the optima in an enlarged bounded region. One way to search for the optima is to randomly generate a large user-specified number of initial points in the bounded region and do a controlled random search (using the \textit{crs2lm} function in the optimization package \textbf{nloptr} \citep{nloptr} in \textit{R}). We recommend controlled random search because it does not rely on the properties of the objective function for global optimization, hence avoiding the issues from nonconvexity and nonsmoothness. Some other direct search methods are also applicable in this case, such as the simplex algorithm \citep{nelder1965simplex} and the differential evolution \citep{storn1997differential}. This procedure is referred to as the \textbf{CRS Minimization} in Section \ref{sim} and Section \ref{empirical}. The procedure is described as follows:
	
	\begin{enumerate}
		\item Initialize the controlled random search by $(\wh{\alpha}^{*}_{(0)}, \wh{c}^{*}_{(0)})$, calculated as follows. Obtain $\wh{\alpha} = (\wh{\alpha}_1,\ldots,\wh{\alpha}_K)^T$ by regressing $Y$ on $Z$ (defined in Step 3 of Section \ref{Common}) under squared error loss with nonnegative constraint. Locate the estimated coefficient with the largest value. Without loss of generality, assume $\wh{\alpha}_1 = \text{max}_{k\in \{1,\ldots,K\}} \wh{\alpha}_k $. 
		Define $\wh{\alpha}^{*}_{(0)} = (1,\wh{\alpha}_2/\wh{\alpha}_1,\ldots,\wh{\alpha}_K/\wh{\alpha}_1)^T$ and estimate $\wh{c}^{*}_{(0)}=\operatorname*{argmin}_c \sum^n_{i=1}\lambda \mathbbm{1}\{Y_i=1,Z_i \wh{\alpha}^{*}_{(0)}  < c\} + (1-\lambda)\mathbbm{1}\{Y_i=0,Z_i\wh{\alpha}^{*}_{(0)} \ge c\} $ by line search. 
		\item Apply controlled random search on an enlarged nonnegative bounded region, e.g. in the following sections we empirically chose the enlarged region as $[0,5]^K$ based on the magnitude of $\wh{\alpha}'s$. Obtain $\wh{\alpha}^* = ( \wh{\alpha}^*_1,\ldots,\wh{\alpha}^*_K)$ and $\wh{c}^*$ from the controlled random search, as estimates to a minimizer of equation \eqref{step3}.
		\item Normalize both the coefficients and threshold estimates by $\sum_{k=1}^K\wh{\alpha}^*_k$, so the coefficient estimates sum to one.
	\end{enumerate}
	
	The controlled random search here does not give an unbiased or efficient estimate of the class probability, even if the resulting scores are scaled to between zero and one. The estimated scores do not need to have a probabilisitic interpretation and the solutions may not be unique, but they are all valid approximations in terms of minimizing the weighted misclassification loss.  
	% Stability is not guaranteed, but empirical results indicate promising performance.
	
	For the theoretical completion of this work, we are able to show the asymptotic optimality of SL with joint thresholding under the weighted misclassification loss by the following theorem:
	%\begin{comment}
	\begin{theorem*}
		Let $S$ represents a random data split that is stochastically independent of the observations, resulting in a training set and a test set of a nonneglible size. For the weighted misclassification loss $L_\lambda(Y,Q(X))$ at a given $\lambda\in (0,1)$, and classifiers in $\mathcal{Q} = \{Q_\theta: \theta \in \Theta_n\}$, where $\theta = (\alpha, c)$, $Q_\theta (x) = \mathbbm{1}\{\sum^K_{k=1} \alpha_k\Psi_k (x) \ge c\}$,
		and $\Theta_n$ is a bounded discretized parameter space of $\theta$, there is
		\begin{align*}
		\mathrm{E}_S R_\lambda(Q_{\hat{\theta}}(P_S^0))&\le \mathrm{E}_S R_\lambda(Q_{\tilde{\theta}}(P_S^0)) + O\bigg(\frac{\text{log}(n)}{\sqrt{n}}\bigg),
		\end{align*}
		where $n$ is sample size of the entire data, $\mathrm{E}_S R_\lambda(Q_{\hat{\theta}}(P_S^0))$ and $\mathrm{E}_S R_\lambda(Q_{\tilde{\theta}}(P_S^0))$ are the risks averaged over data splits for the SL with joint thresholding and an oracle classifier respectively.
	\end{theorem*}
	%\end{comment}
	The web appendix (A3) provides details of the notations and a proof of the theorem.
	
	The oracle classifier is the best classifier among all classifiers in the form of $Q_\theta (x)$ estimated from the data, such that it is closest in distance to the unparametrized true minimizer of the weighted misclassifiction risk. The theorem is saying that under the weighted misclassification loss, SL with parameters estimated from equation \eqref{step3} asymptotically converges to the oracle in terms of average risk as the sample size grows to infinity. 
	
	\section{Simulations}\label{sim}
	
	The key objective of our simulation study is to compare our approach to the conditional thresholding approach in minimizing the out-of-sample weighted misclassification risk when the true model can or cannot be easily recovered. We set it up so that we know the optimal classification rule and can compare with it the rules obtained from conditional and joint thresholdings. We first simulate two datasets $\mathcal{D}_1$ and $\mathcal{D}_2$, each of size $n = 10^4$. For weighted misclassification loss at fixed $\lambda\in (0, 1)$, we apply conditional thresholding, two-step minimization and CRS minimization to $\mathcal{D}_1$, and use them to estimate three classification rules. For each rule and each value of $\lambda$, the corresponding out-of-sample weighted misclassification risk is calculated by applying the derived rule on $\mathcal{D}_2$.
	
	Two data generating mechanisms, adapted from Kang and Schafer \cite{kang2007demystifying}, are considered.
	For each one, the observed binary outcome $Y_i$ is generated from an underlying score $\tilde{Y}_i$ via $Y_i = \mathbbm{1}(\tilde{Y}_i \ge c)$, where the cutoff $c$ is chosen to guarantee a $30\%$ prevalence. The underlying score is generated by $\tilde{Y}_i = b_0 + U_ib +\epsilon_i$, where $b_0 = 210$, $b = (27.4, 13.7, 13.7, 13.7)^T$; for $i=1,\ldots,n$, $\epsilon_i\sim N(0, 100^2)$, and $U_i = (U_{i1}, U_{i2}, U_{i3}, U_{i4}) \sim N(0, I_{4\times 4})$.
	
	In the first setting we treat $U$ as observed covariates and uses $(Y, U)$ to derive classification rules. In the second setting, instead of using $U$, we observe $X = g(U)$, where $g: \rm I\!R^4 \rightarrow \rm I\!R^4$ is given by $g_1(u) = \exp(u_1/2)$, $g_2(u) = u_2/(1+\exp(u_1)) + 10$, $g_3(u) = (u_1u_3/25 + 0.6)^3$, and $g_4(u) = (u_2 + u_4 + 20)^2$. More detailed rationale for this parameterization is given in Kang and Schafer \cite{kang2007demystifying}.
	
	%uses $(Y, X)$, where $X_i = (X_{i1}, X_{i2},X_{i3}, X_{i4})^T$ with $X_{i1} = \exp(U_{i1}/2)$, $X_{i2} = U_{i2}/(1+\exp(U_{i1})) + 10$, $X_{i3} = (U_{i1}U_{i3}/25 + 0.6)^3$, and $X_{i4} = (U_{i2} + U_{i4} + 20)^2$. 
	
	Earlier, the true probability score is $\Psi_0(U) = P(Y = 1\mid U) = P(\epsilon \ge c-U\beta\mid U)$. The optimal Bayes classification rule based on the true probability score is $Q_0(U) = \mathbbm{1}\{\Psi_0(U)\ge 1-\lambda\}$, which provides a reference standard for assessing classification accuracy. 
	The out-of-sample risk for the optimal Bayes classification rule is approximated by $$\wh{R}_\lambda(\mathcal{D}_2;Q_0)=   \frac{1}{n}\sum_{(U_i,Y_i)\in\mathcal{D}_2} \lambda \mathbbm{1}\{Q_0(U_i)=0,Y_i=1\}+(1-\lambda)\mathbbm{1}\{Q_0(U_i)=1,Y_i=0\}.$$ We therefore use $\wh{R}_\lambda(\mathcal{D}_2;Q_0)$ as the reference for relative differences displayed in Table \ref{CVRiskSim}. The relative difference of estimated weighted misclassification risk for an estimated classification rule $\wh{Q}(\cdot)$ at penalty $\lambda$ is defined as $\{\wh{R}_\lambda(\mathcal{D}_2;\wh{Q})-\wh{R}_\lambda(\mathcal{D}_2;Q_0)\}/\wh{R}_\lambda(\mathcal{D}_2;Q_0)$. 
	
	We consider using four and eight candidate algorithms for SL. In the case of $K=4$, the SL library $\mathscr{L}$ includes random forest \citep{breiman2001random}, logistic regression, generalized additive model \citep{hastie1990generalized} and CART \citep{breiman1984}. For $K=8$, four additional candidate algorithms are added to $\mathscr{L}$: 10 nearest neighbors, generalized boosting \citep{friedman2001greedy}, support vector machine \citep{hsu2003practical} and bagging classification \citep{breiman1996bagging}. Logistic regression and generalized additive model are fitted using maximum likelihood without penalization; we use linear main effect terms for logistic regression and quadratic main effect splines for generalized additive model.
	
	Relative risk differences for each of the SL classification methods discussed above are summarized in Table \ref{CVRiskSim}. In the first setting when covariates $U$ are used to develop the classification rule, conditional thresholding and joint thresholding rules (two-step and CRS minimization) have, as expected, similar out-of-sample performances, differing from the optimal classification rule by less than $2\%$ relatively across all values of $\lambda$. In the second setting, where classification is based on $X$, joint thresholding clearly outperforms conditional thresholding for most $\lambda$ values. Relative to the optimal Bayes classification rule based on the true probability score $\Psi_0(U)$, conditional thresholding is worse by 5 to 23$\%$ across $\lambda$, while joint thresholding differs by 2.3$\%$ or less. 
	
	\section{Applications to HIV and Breast Cancer Data}\label{empirical}
	
	In this section we illustrate our proposed methods on Kenyan clinical HIV data and Wisconsin diagnostic breast cancer data. We use the same SL libraries with four and eight learners as in the simulations. The first data set was used in Liu et al. \citep{liu_improved_2017} and the second data set is available on the UCI machine learning data repository \citep{Lichman:2013}. 
	
	The Kenyan HIV data were derived from three studies conducted at the Academic Model Providing Access to Healthcare (AMPATH) in Eldoret, Kenya: (i) The ``Crisis" study (n=191) \citep{mann2013antiretroviral}, conducted in 2009-2011 to investigate the impact of the 2007-2008 post-election violence in Kenya on ART failure and drug resistance; (ii) The ``second-line" study (n=394) \citep{secondline}, conducted in 2011-2012 to investigate ART failure and drug resistance upon second-line ART; and (iii) The ``TDF" study (n=333) \citep{brooks2016treatment}, conducted in 2012-2013 to investigate the impact of WHO guidelines changes to TDF-based first-line ART on HIV treatment failure and drug resistance. The data include covariate information on age, gender, nadir CD4 count, CD4 count, CD4 percent, adherence to ART, time since starting current ART, slope of CD4 percent progression and the outcome VL. Our interest is to develop a classification rule to predict HIV virological failure (VL $>$ 1000 copies/ml) using the covariates. Prevalence of HIV virological failure is 15.7$\%$. We use data on 899 complete cases for our analysis. 
	
	The Wisconsin diagnostic breast cancer data has 569 observations. The outcome is histologically confirmed diagnosis of breast cancer as either benign or malignant. Prevalence of malignancy is 37$\%$. For each person, a digital image from FNA was generated to contain the most abnormal appearing cells, and ten image features were calculated for each of the selected cell: radius, texture, perimeter, area, smoothness, compactness, concavity, concave points, symmetry and fractal dimension. For each image feature, three measures were then computed per person: the mean, standard error, and ``worst" (mean of three largest values), resulting in 30 covariates for analysis.
	
	%Pima Indians diabetes data was collected before 1990 by National Institute of Diabetes and Digestive and Kidney Diseases from Pima Indian women near Pheonix, Arizona, under enrollment criteria described in \citet{smith1988using}. It was of a concern to study this population due to its high incidence rate of diabetes \citep{bennett1971diabetes}, and the study goal was to forecast the onset of non-insulin-dependent diabetes mellitus within a five-year period. The data has 768 observations, outcome class value 1 is interpreted as ``tested positive for diabetes" and has prevalence $34.9\%$. \citet{smith1988using} defined the diabetes pedigree function based on patients' close relatives to measure the expected genetic influence on patients' eventual diabetes risk. Covariates used in analysis are: number of times pregnant, plasma glucose concentration at 2 hours in an oral glucose tolerance test, diastolic blood pressure, triceps skin fold thickness, 2-hour serum insulin, body mass index, diabetes pedigree function and age.
	
	For both data, continuous covariates are rescaled to have a mean of zero and a standard deviation of one, binary covariates are rescaled to -1 and 1. The performance measure is the cross-validated weighted misclassification risk, which aggregates over all iterations of a 10-fold cross validation.

	Table \ref{CVRiskPrc} shows that the two-step and CRS minimization have similar or lower cross-validated weighted misclassification risks at penalty weight $\lambda$ = 0.2, 0.5 and 0.8 compared to the conditional thresholding approach for both data illustrations. In addition, the difference in cross-validated risks between conditional and joint thresholding is smaller for eight learners compared to four learners. For breast cancer malignancy prediction at $\lambda=0.2$ under the scenario of four candidate learners, the estimated risks of conditional and joint thresholding differ by as much as 2.3$\%$, which is a relative difference of 62$\%$. 
	
	Table \ref{estates table} displays the estimated thresholds and coefficients using eight candidate learners for the Kenyan HIV data under false-negative penalty $\lambda = $ 0.2 and 0.8, for conditional thresholding and our proposed approaches. In our analysis, both conditional thresholding and two-step minimization derive coefficients by nonnegative least squares linear regression and normalize to sum to one, hence their coefficient estimates are the same and stay fixed across $\lambda$. For CRS minimization, the coefficient estimating procedure involves optimizing a function of $\lambda$, so the estimated values may vary by $\lambda$. Threshold estimation depends on $\lambda$ for all approaches. Conditional thresholding and two-step minimization share the same estimated risk scores for all $\lambda$ values, but their threshold estimates are different due to the different thresholding methods.
	
	Figure \ref{2by2plot} provides the cross-validated risk comparison in a finer scale for the Kenyan HIV data and the Wisconsin diagnostic breast cancer data, considering both four and eight candidate learners. Differences in cross-validated risks between conditional and joint thresholding across $\lambda$ are larger for fewer candidate learners in the SL library, and joint thresholding approaches generally have better cross-validated risks compared with the conditional thresholding.  
	
	Figure \ref{wdbcplot} illustrates why misclassification differs between the thresholding methods. The vertical lines in Figure \ref{wdbcplot} indicate threshold values estimated at $\lambda=0.8$ based on SL, cross-validated SL, and $Z\hat{\alpha}$ predictions correspondingly from left to right. Under conditional thresholding, the rule is conditional on SL prediction, which may be subject to over-fitting (the first panel). By contrast, two-step and CRS minimization use cross-validated prediction for selecting the threshold, and do this within the SL (the third panel), which has a similar distribution as the cross-validated prediction of the entire SL (middle panel). As expected from risk analysis of the methods, estimated thresholds are very different for conditional and joint thresholding. In Figure \ref{wdbcplot}, threshold estimates from joint thresholding and cross-validating the entire SL are very similar, which further explains the out-performance of joint thresholding compared to conditional thresholding. Nonetheless, threshold calculation by cross-validating the entire SL requires substantially more computation time and complexity in both estimation and evaluation, and may result in over-cross-validating the data when sample size is not sufficiently large. Furthermore, none of the estimated thresholds is close to the optimal Bayes threshold of $1-\lambda=0.2$, implying that all the three types of derived predictions from SL do not highly match the unknown underlying true classification probability, indicating the neccessity to do optimal classification rule derivation.

	%%%%%%%%%%%%%%%%%%%%%%%%%%%%
	
	\section{Summary and Discussion}
	
	Weighted misclassification risk is often used to evaluate predictions when false-positives or false-negatives need to be prioritized differently. However, inference and rule derivation using weigthed misclassification risk is less common due to the difficulties in numerical computations. For binary classification using SL, we aimed to optimally estimate both the threshold and ensemble weights associated with candidate learners by minimizing the weighted misclassification risk. Through simulations and data examples, we showed that the conditional thresholding may generate sub-optimal classification rules. We proposed two options for joint thresholding, both embedding the threshold estimation procedure within SL, and showed that our proposal performs similarly or outperforms the conditional thresholding approach in terms of determining the optimal classification rules. 
	
	Our method presents a new way to estimate the classification rule. We show that the rules developed under either two-step or CRS minimization generally have lower error rates compared to conditional thresholding. From the comparison of density curves in Figure \ref{wdbcplot}  among SL prediction, cross-validated SL prediction and $Z\alpha$ within SL in data application, we can see that the conditional thresholding tends to overfit the data, which results in an incorrectly estimated risk. Therefore, we expect the actual risk from conditional thresholding to be closer to the true risk under two situations: first, when the training data distribution can well represent the true underlying data distribution, in which case there would not be much difference in distribution between the training and the test sets; second, when the ensemble risk score can well discriminate the data for both the training and the test set, as reflected by Figure \ref{2by2plot}, the difference between conditional and joint thresholding becomes smaller with more candidate learners in the SL library.
	
	Our work also provides a general framework for using ensemble learners for binary classification. Although we consider weighted loss functions as in \eqref{step3}, our method has the potential to be extended to more general threshold-based classifications, and numerical optimization method should change accordingly based on the nature of the threshold-based classification loss. We expect similar property to hold for other classification problems that involve threshold estimation, when the classification loss is a measurable function of classifiers.
	
	Furthermore, from Figure \ref{wdbcplot}, we anticipate the performance of our method to be comparable to threshold estimation based on cross-validated SL predictions. This is important for settings where computational complexity is high, or the size of the SL library and data are large, as threshold esimation based on cross-validated SL predictions may require substantial amount of computations, increasing complication and difficulty to method evaluation. 
	
	Code for the analysis was written in R \citep{R} and is available in the web appendix.
	
	%\end{doublespacing}

	\newpage
	
	\begin{table}[h]
		\centering
		\caption{Out-of-sample weighted misclassification risk of the optimal Bayes classification rule ($\%$) in the first row, and relative difference in out-of-sample weighted misclassification risk ($\%$)  at $\lambda = $ 0.2, 0.5 and 0.8 for simulation studies described in Section \ref{sim}, stratified by estimation methods and number of learners in SL library. Relative difference for a derived classification rule $\wh{Q}(\cdot)$ at $\lambda$ is $(\wh{R}_\lambda(\mathcal{D}_2;\wh{Q})-\wh{R}_\lambda(\mathcal{D}_2;Q_0))/\wh{R}_\lambda(\mathcal{D}_2;Q_0)$, where $Q_0$ is the optimal Bayes classification rule based on the true probability score $\Psi_0(U)$.}
		\label{CVRiskSim}
		\begin{tabular}{llccccccc}
			\hline
			&                                                                            & \multicolumn{3}{c}{4 Learners}                                                &                      & \multicolumn{3}{c}{8 Learners}                                                \\ \cline{3-5} \cline{7-9} 
			&                                                                            & \multicolumn{3}{c}{$\lambda$}                                                 &                      & \multicolumn{3}{c}{$\lambda$}                                                 \\
			&                                                                            & 0.2                     & 0.5                      & 0.8                      &                      & 0.2                     & 0.5                      & 0.8                      \\ \cline{3-5} \cline{7-9} 
			\multicolumn{1}{l}{} & \multicolumn{1}{l}{$\wh{R}_\lambda(\mathcal{D}_2;Q_0)(\%)$} & \multicolumn{1}{l}{6.1} & \multicolumn{1}{l}{14.8} & \multicolumn{1}{l}{12.9} & \multicolumn{1}{l}{} & \multicolumn{1}{l}{6.1} & \multicolumn{1}{l}{14.8} & \multicolumn{1}{l}{12.9} \\
			Simulation 1         & Conditional Thresholding                                                   & 0.0                     & 0.0                      & 1.6                      &                      & 0.0                     & 0.0                      & 1.6                      \\
			& Two-Step Minimization                                                      & 0.0                     & 0.0                      & 0.8                      &                      & 0.0                     & 0.0                      & 1.6                      \\
			& CRS Minimization                                                           & 0.0                     & 0.0                      & 0.8                      &                      & 0.0                     & 0.0                      & 1.6                      \\
			Simulation 2         & Conditional Thresholding                                                   & 16.4                    & 12.8                     & 11.6                     &                      & 8.2                     & 6.1                      & 9.3                      \\
			& Two-Step Minimization                                                      & 0.0                     & 0.0                      & 2.3                      &                      & 0.0                     & 0.0                      & 2.3                      \\
			& CRS Minimization                                                           & 0.0                     & 0.0                      & 2.3                      &                      & 0.0                     & 0.0                      & 2.3                     
			\\	\hline
		\end{tabular}
	\end{table}

	\begin{table}[h]
		\centering
		\caption{ Cross validated weighted misclassification risk ($\%$) at $\lambda = $ 0.2, 0.5 and 0.8 for two data examples, stratified by estimation methods and number of learners in SL library.}
		\label{CVRiskPrc}
		\begin{tabular}{llccccccc}
			\hline
			&                          & \multicolumn{3}{c}{4 Learners} &  & \multicolumn{3}{c}{8 Learners} \\ \cline{3-5} \cline{7-9} 
			&                          & \multicolumn{3}{c}{$\lambda$}  &  & \multicolumn{3}{c}{$\lambda$}  \\
			&                          & 0.2      & 0.5      & 0.8      &  & 0.2      & 0.5      & 0.8      \\ \cline{3-5} \cline{7-9} 
			Kenyan HIV    & Conditional Thresholding & 3.9      & 8.5      & 8.7      &  & 3.5      & 7.8      & 8.9      \\
			& Two-Step Minimization    & 3.2      & 7.0      & 8.7      &  & 3.3      & 7.6      & 8.8      \\
			& CRS Minimization         & 3.2      & 7.0      & 8.5      &  & 3.3      & 7.6      & 8.7      \\
			Breast Cancer & Conditional Thresholding & 3.7      & 2.5      & 1.2      &  & 1.8      & 1.6      & 0.9      \\
			& Two-Step Minimization    & 1.4      & 1.8      & 0.9      &  & 1.2      & 1.4      & 0.8      \\
			& CRS Minimization         & 1.4      & 1.8      & 0.8      &  & 1.2      & 1.4      & 0.9     \\
			\hline
		\end{tabular}
	\end{table}

	\begin{table}[h]
		\centering
		\caption{ Estimates of $(\alpha, c)$ at $\lambda = $ 0.2 and 0.8 for Kenyan HIV data, under estimation methods: Conditional Threshold (CT), Two-Step Minimization (2-Step) and CRS Minimization (CRS).\\ $*$ indicates that $\wh{\alpha}$ are the same for CT and 2-Step regardless of $\lambda$.}
		\label{estates table}
		\begin{tabular}{lccccccc}
			\hline
			& \multicolumn{3}{c}{$\lambda$ = .2}       &  & \multicolumn{3}{c}{$\lambda$ = .8}       \\
			& $\text{CT}^*$ & $\text{2-Step}^*$ & CRS  &  & $\text{CT}^*$ & $\text{2-Step}^*$ & CRS  \\ \cline{2-4} \cline{6-8} 
			$\hat{\alpha}_{\text{random forest}}$        & 0.11          &                   & 0.11 &  &               &                   & 0.04 \\
			$\hat{\alpha}_{\text{logistic regression}}$  & 0             &                   & 0    &  &               &                   & 0.19 \\
			$\hat{\alpha}_{\text{quadratic splines}}$    & 0.42          &                   & 0.42 &  &               &                   & 0.16 \\
			$\hat{\alpha}_{\text{CART}}$                 & 0             &                   & 0    &  &               &                   & 0.33 \\
			$\hat{\alpha}_{\text{10-NN}}$                & 0.20          &                   & 0.20 &  &               &                   & 0.01 \\
			$\hat{\alpha}_{\text{generalized boosting}}$ & 0.27          &                   & 0.27 &  &               &                   & 0.06 \\
			$\hat{\alpha}_{\text{SVM}}$                  & 0             &                   & 0    &  &               &                   & 0.12 \\
			$\hat{\alpha}_{\text{Bagging}}$              & 0             &                   & 0    &  &               &                   & 0.08 \\
			$\hat{c}$                                    & 0.62          & 0.73              & 0.73 &  & 0.16          & 0.19              & 0.18 \\ \hline
		\end{tabular}
	\end{table}

	\newpage
	
	\begin{figure}[ht]
		\centering
		\includegraphics[scale = 0.7]{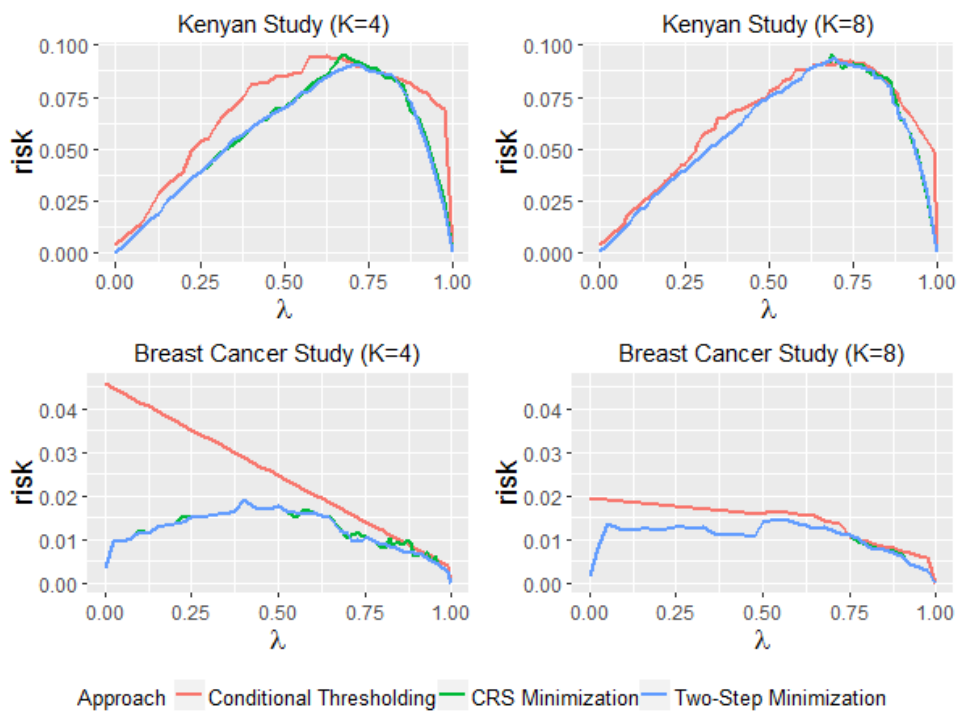}  
		\caption{Comparison of cross-validated weighted misclassification risk as a function of $\lambda$ for the two data applications, stratified by number of learners in SL library.}
		\label{2by2plot}
	\end{figure}
	
	\begin{figure}[ht]
		\centering
		\includegraphics[scale = 0.5]{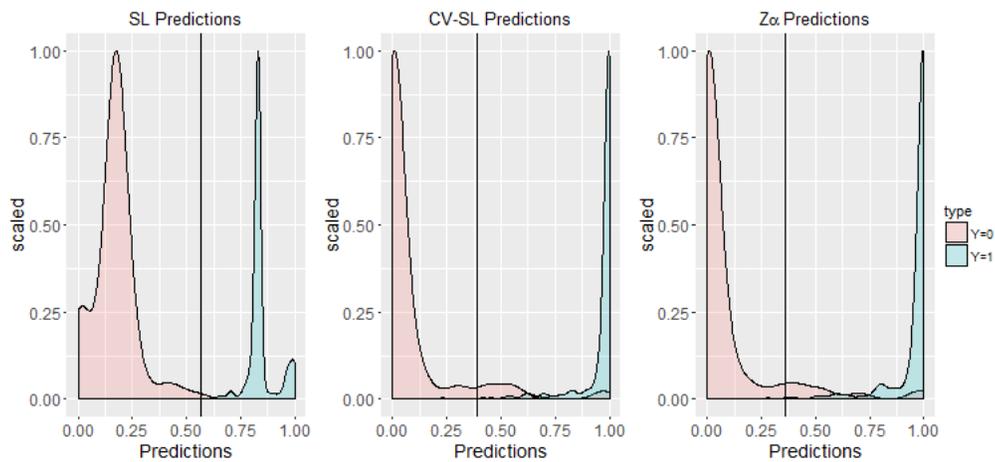}  
		\caption{Density curve of SL prediction, 10-fold cross-validated prediction of SL (CV-SL) and combined cross-validated prediction within SL ($Z\alpha$) for breast cancer data, using least squares regression under nonnegative and sum to one constraints on the coefficients for combining the eight library prediction scores. Vertical line represents estimated threshold based on corresponding type of prediction at $\lambda=0.8$.}
		\label{wdbcplot}
	\end{figure}

	\clearpage
	
	\section*{Web Appendix}
	
	%\date{\today}
	%\author{Yizhen Xu \and Tao Liu \and Rami Kantor \and Ann Mwangi \and Michael J. Daniels \and Joseph W.\ Hogan}
	%\section{Introduction}
	\section*{A1. Application to SECOM Data}
	
	The following analysis uses the large p SECOM data set \citep{dheeru_uci_2017} from the UCI machine learning repository. Each sample in the SECOM data set represents a single production entity from a modern semi-conductor manufacturing process; the outcome represents a pass/fail yield for in-house line testing and the features are measured signals from the monitoring system. The original data set has 1567 samples, 591 features, and 104 fails in its outcome. The number of features in the data set is over 1/3 of the sample size, and regular small p methods such as linear regression and additive splines do not converge for this data. 
	
	We cleaned the data so that all the features have more than one unique values and each feature has either no or at least 30 missingness. After accounting for missingness by creating missing indicators and filling in 0’s for missing measurements, there are 1436 samples, 484 features, and 100 fails. We used three candidate learners for this large p application: random forest, lasso, and leekasso. Leekasso does linear model using the top 10 most significant predictors from fitting univariate models of each covariate. These three candidate learners were chosen from convergence and computation speed considerations. 
	
	From the following cross-validated risk plot we show that our proposed joint thresholding can be applied to large-p problems with properly selected candidate learners and our method still performs better than the conditional thresholding. 
	
	\begin{figure}[H]
		\centering
		\includegraphics[scale = 0.6]{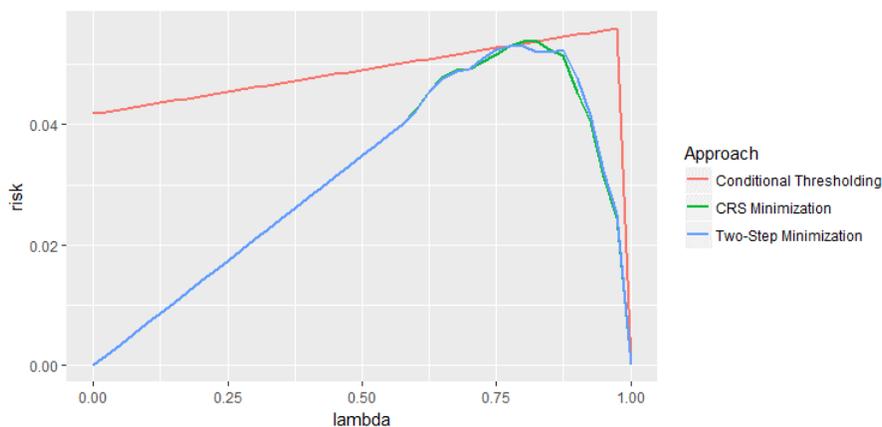}  
		\caption{Comparison of cross-validated weighted misclassification risk as a function of $\lambda$ for the SECOM data application between conditional thresholding and joint thresholding (CRS and two-step minimization), under the following candidate learners in SL library: random forest, lasso, and leekasso.}
	\end{figure}
	
	\section*{A2. Comparison with Candidate Learners -- SECOM Data}
	
	Because candidate learners return risk estimations, which are continuous scores, each individual candidate learner would need a threshold in order to show its classification performance under the weighted misclassification loss. For a single learner, the joint thresholding method reduces to deriving a threshold based on the cross-validated predictions of that learner. To actually see how much additional gain the ensemble learner is achieving, we used the SECOM data example again. The 10-fold cross-validated risks are shown in Figure 2. Table 1 has the values of the curves in Figure 2 specifically at $\lambda = 0.1,0.2,\ldots,0.9$.
	
	From Figure 2 we can see that the performance of SL with joint thresholding is similar to its best candidate learner, which is the random forest with joint thresholding in this case.
	
	\begin{figure}[ht]
		\centering
		\includegraphics[scale = 0.8]{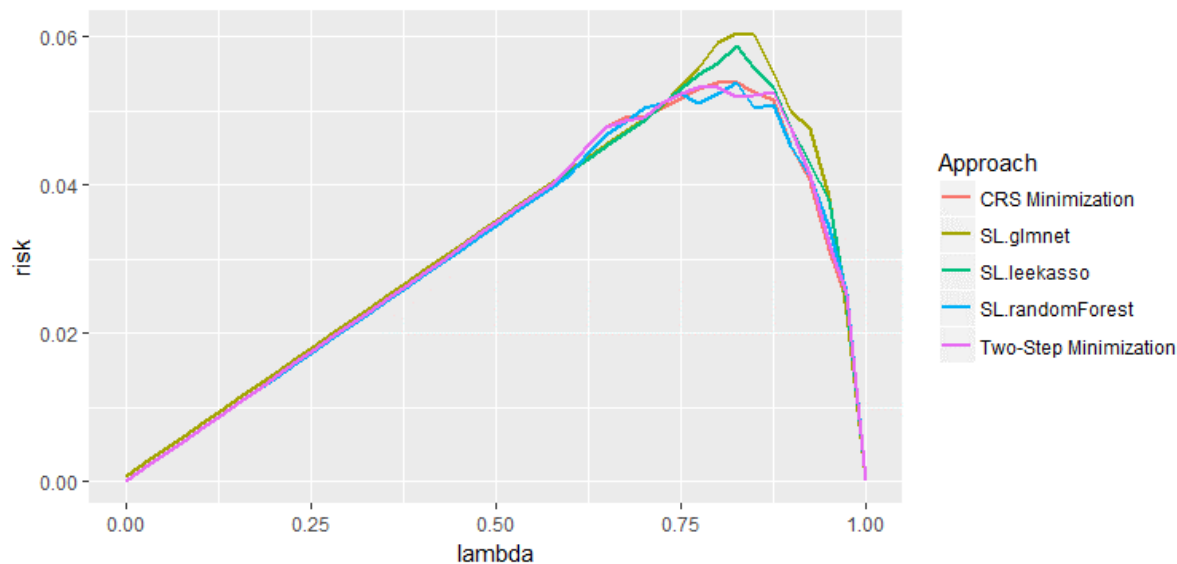}  
		\caption{Comparison of cross-validated weighted misclassification risk with joint thresholding as a function of $\lambda$ for the SECOM data application among the SL (CRS minimization and two-step minimization) and the individual candidate learners (random forest, lasso, and leekasso).}
	\end{figure}
	
	\begin{table}[h]
		\centering
		\caption{Cross-validated weighted misclassification risk with joint thresholding at $\lambda$ from 0.1 to 0.9 for the SECOM data application, stratified by estimation methods}
		\begin{tabular}{cccccccccc}
			\hline
			$\lambda$                      & 0.1   & 0.2   & 0.3   & 0.4   & 0.5   & 0.6   & 0.7   & 0.8   & 0.9   \\
			SL Joint Thresholding Two-Step & 0.007 & 0.014 & 0.021 & 0.028 & 0.035 & 0.042 & 0.049 & 0.053 & 0.048 \\
			SL Joint Thresholding CRS      & 0.007 & 0.014 & 0.021 & 0.028 & 0.035 & 0.042 & 0.049 & 0.054 & 0.045 \\
			RF+CV threshold                & 0.007 & 0.014 & 0.021 & 0.028 & 0.034 & 0.041 & 0.05  & 0.052 & 0.045 \\
			lasso+CV threshold             & 0.008 & 0.014 & 0.021 & 0.028 & 0.035 & 0.042 & 0.049 & 0.059 & 0.05  \\
			leekasso+CV threshold                    & 0.007 & 0.014 & 0.021 & 0.028 & 0.035 & 0.042 & 0.049 & 0.056 & 0.048 \\ \hline
		\end{tabular}
	\end{table}

	\newpage
	
	\section*{A3. Theoretical Justification}\label{theory}
	
	Our theoretical justification follows a similar road map as in \citet{van2007super} and \citet{vaart_oracle_2009}. We first introduce and revise some notations.
	
	Given false negative penalty $\lambda\in (0,1)$, let us denote the true risk minimizer is the classifier by
	$$Q_0 = \operatorname*{argmin}_Q E_{(Y,X)} L_\lambda(Y, Q(X)) = \operatorname*{argmin}_Q R_\lambda (Q)$$
	where $$L_\lambda(Y,Q(X))=\lambda \mathbbm{1}\{Q(X)=0,Y=1\}+(1-\lambda)\mathbbm{1}\{Q(X)=1,Y=0\}$$ and $$R_\lambda (Q) = \int L_\lambda(y, Q(x)) dP(x, y).$$ 
	In the above expression, notice that $P$ is the unknown true distribution over the outcome $Y$ and covariates $X$.
	
	The distance between any two classifiers $Q_1$ and $Q_2$ is defined as the absolute difference between their risks, 
	$$d_\lambda(Q_1, Q_2) = |R_\lambda(Q_1) - R_\lambda(Q_2)|.$$
	
	For SL, there is a collection of scores from the $K$ candidate learners, $(\Psi_1, \ldots , \Psi_K)$. We parameterize the classification rule as $Q_\theta(\cdot)$, where $\theta = (\alpha, c)$, $\alpha$ is the vector of coefficients that linearly combines the $K$ scores, and $c$ is the threshold to be applied to the combined score for binary classification,
	$$Q_\theta(x) = \mathbbm{1}\{\sum^K_{k=1} \alpha_k \Psi_k(x) \ge c\}$$
	
	Without loss of generality, we assume the coefficients in $\alpha$ to be constrained as $|\alpha_i|\in [0,1], i=1,\ldots,K$ and $\sum_{i=1}^{K} \alpha_i = 1$.  Denote the range of score predictions as $\Omega$, which is assumed to be bounded, the parameter space of $\theta$ can be written as $\Theta = [0,1]^K \times \Omega$. We consider a grid $\Theta_n$ of $\theta$ values in the bounded parameter space $\Theta$, and let $K(n) = \#\Theta_n$ be the number of grid points such that $K(n)\le n^q$ for some constant $q < \infty$. We then consider $\mathcal{Q} = \{Q_\theta: \theta \in \Theta_n\}$ as a collection of candidate classifiers. 
	
	Next, we formalize cross-validation as in \citet{vaart_oracle_2009} and describe our estimators accordingly following these notations. Let $S = (S_1,\ldots,S_n)\in\{0,1\}^n$ be a random vector independent of data samples $X_1,\ldots,X_n$; $S_i=0$ indicates that the sample $X_i$ belongs to the training set, otherwise the test set. We can then define the empirical distributions of the training and the test set by 
	$$P_S^j = \frac{1}{n^j} \sum_{i:S_i=j} \delta_{X_i},\qquad n^j = \sum_{i=1}^n \mathbbm{1}\{S_i=j\}, \qquad j=0,1$$
	where $ \delta_{X_i}(x) = \mathbbm{1}\{x\ge X_i\}$ for $i=1,\ldots,n$.
	
	$Q_\theta(P_S^0)$ is a classifier estimated from the training set,
	\begin{align*}
	Q_\theta(P_S^0) (x) &= \mathbbm{1}\{\sum^K_{k=1} \alpha_k\Psi_k(P_S^0)(x) \ge c\}.
	\end{align*}
	
	An oracle selector of $\theta$ is one that its corresponding classifier estimated from the training set minimizes the risk on the unknown distribution $P$ averaged over the splits, 
	\begin{align*}
	\tilde{\theta} &= \operatorname*{argmin}_{\theta\in \Theta_n} \mathrm{E}_S \bigg[\int L_\lambda(y, Q_\theta(P_S^0) (x)) dP(x, y)\bigg]\\
	&= \operatorname*{argmin}_{\theta\in \Theta_n} \mathrm{E}_S \bigg[R_\lambda(Q_\theta(P_S^0)) - R_\lambda( Q_0)\bigg].
	\end{align*}
	The oracle classifier is $Q_{\tilde{\theta}}$, and among all classifiers estimated from the training set, it is the closest in distance to the true risk minimizer $Q_0$.
	
	Cross validation replaces $P$ by the test set distribution $P_S^1$, and uses $\theta_n$ as the cross-validated selector of $\theta$, 
	\begin{align*}
	\theta_n &= \operatorname*{argmin}_{\theta\in \Theta_n} \mathrm{E}_S \bigg[\int L_\lambda(y, Q_\theta(P_S^0) (x)) dP_S^1(x, y)\bigg].
	\end{align*}
	
	The realization of $\theta_n$ is the parameter estimate for SL with joint thresholding as presented in the paper,
	\begin{align*}
	\theta_n &\equiv \operatorname*{argmin}_{\theta\in \Theta_n}  \frac{1}{n}\sum^n_{i=1} \lambda \mathbbm{1}\{\sum^K_{k=1} \alpha_kZ_{ik}<c,Y_i=1\}+(1-\lambda)\mathbbm{1}\{\sum^K_{k=1} \alpha_kZ_{ik}\ge c,Y_i=0\}
	\end{align*}
	where  $Z = \{\wh{\Psi}_{k, T(d)} (X_{V(d)}), k = 1,\ldots,K, d = 1,\ldots,D\}$ is the stacked cross-validated predictions as mentioned in Sections 3 and 4. In Section 4, the $(\wh{\alpha}, \wh{c})$ at the end of our proposed procedure is the solution or approximation to the $\theta_n$.
	
	The goal of this section is to theoretically prove that, the risk averaged over data splits for classifier estimator $Q_{\theta_n}$ asymptotically converges to that of the oracle classifier $Q_{\tilde{\theta}}$. \citet{vaart_oracle_2009} (Theorem 2.3) established the following inequality for the risk averaged over data splits between the cross-validated selector and the oracle selector.
	
	\begin{theorem}
		\citep{vaart_oracle_2009} For $Q\in \mathcal{Q}$, let $(M(Q),v(Q))$ be a Bernstein pair for the measurable function $z\mapsto L(z,Q)$ and assume that $R(Q) = \int L(z,Q) dP(z)\ge 0$ for every $Q\in\mathcal{Q}$. Then for any $\delta>0$ and $1\le p \le 2$,
		\begin{align*}
		\mathrm{E}_S R(Q_{\hat{\theta}}(P_S^0))&\le (1+2\delta)\mathrm{E}_S R(Q_{\tilde{\theta}}(P_S^0))+ (1+\delta)\mathrm{E}_S \bigg[ \frac{16}{(n^1)^{1/p}}\bigg]\\
		& \times \log(1+\#\mathcal{Q}) \operatorname*{sup}_{Q\in \mathcal{Q}} \bigg[ \frac{M(Q)}{(n^1)^{1-1/p}} + \bigg(\frac{v(Q)}{R(Q)^{2-p}}\bigg)^{1/p} \bigg(\frac{1+\delta}{\delta}\bigg)^{2/p-1}\bigg].
		\end{align*}
	\end{theorem}
	
	Recall that for a measurable function $f:\mathcal{X} \rightarrow \mathcal{R}$, $(M(f), v(f))$ is a pair of Bernstein numbers if $$M(f)^2P\bigg(e^{|f|/M(f)}-1-\frac{|f|}{M(f)}\bigg)\le \frac{1}{2} v(f).$$ And it was shown in \citet{vaart_oracle_2009} that if $f$ is uniformly-bounded, then $(||f||_{\infty}, 1.5Pf^2)$ is a pair of Bernstein numbers.
	
	\begin{lemma}
		For a weighted misclassification loss $(x,y)\mapsto L_\lambda(y,Q(x))$, where $Q: x\mapsto \{0,1\}$ and $\lambda\in(0,1)$, its Bernstein pairs $(M(Q),v(Q))$ satisfy
		$$M(Q) = \operatorname*{max}\{\lambda,1-\lambda\}$$ and
		$$v(Q) = \lambda^2 P(Q(X)=0,Y=1) + (1-\lambda)^2 P(Q(X)=1,Y=0).$$
		Furthermore, $v(Q)\le 1.5\times \operatorname*{max}\{\lambda,1-\lambda\} \times R_\lambda (Q)$
	\end{lemma}
	\begin{proof}
		The loss function $L_\lambda(y,Q(x))=\lambda\mathbbm{1}\{Q(x)=0,y=1\}+(1-\lambda)\mathbbm{1}\{Q(x)=1,y=0\}$ has range $\{0,\lambda,1-\lambda\}$ and hence is bounded by $\operatorname*{max}\{\lambda,1-\lambda\}$. 
		
		By definition, risk can be written as
		$$R_\lambda(Q) = \lambda P(Q(X)=0,Y=1)+(1-\lambda) P(Q(X)=1,Y=0).$$
		Hence, the second moment of the loss function has the following form and upper bound,
		\begin{align*}
		& E_{(x,y)}\bigg[L^2(y,Q(x))\bigg] \\
		&= E_{(x,y)}\bigg[\lambda^2\mathbbm{1}\{Q(x)=0,y=1\}+(1-\lambda)^2\mathbbm{1}\{Q(x)=1,y=0\}\bigg] \\
		&= \lambda^2 P(Q(X)=0,Y=1) + (1-\lambda)^2 P(Q(X)=1,Y=0)\\
		&\le  \operatorname*{max}\{\lambda,1-\lambda\}\times R_\lambda(Q).
		\end{align*}
		Therefore,
		$$v(Q) = 1.5 E_{(x,y)}\bigg[L^2(y,Q(x))\bigg] \le 1.5\times \operatorname*{max}\{\lambda,1-\lambda\} \times R_\lambda (Q).$$
	\end{proof}
	
	We apply theorem 1 \citep{vaart_oracle_2009} with the lemma, so that for the weighted misclassification loss, we have
	
	\begin{theorem}
		For classifiers in $\mathcal{Q} = \{Q_\theta: \theta \in \Theta_n\}$ and the weighted misclassification loss $L_\lambda(Y,Q(X))$ at a given $\lambda\in (0,1)$, there is
		\begin{align*}
		\mathrm{E}_S R_\lambda(Q_{\hat{\theta}}(P_S^0))&\le \mathrm{E}_S R_\lambda(Q_{\tilde{\theta}}(P_S^0)) + O\bigg(\frac{\text{log}(n)}{\sqrt{n}}\bigg).
		\end{align*}
		
	\end{theorem}
	\begin{proof}
		Write $C = \operatorname*{max}\{\lambda,1-\lambda\}$. Recall that we assumed the cardinality of $\Theta_n$ as $K(n)\le n^q$, and by definition, $\#\mathcal{Q} = \#\Theta_n$. Hence, we have
		$\log(1+\#\mathcal{Q}) \le q \text{log}(n).$
		Therefore, applying the inequality in Lemma 1 to Theorem 1, for any $\delta>0$ and $1\le p \le 2$, there is
		\begin{align*}
		\mathrm{E}_S R_\lambda(Q_{\hat{\theta}}(P_S^0))\le & (1+2\delta)\mathrm{E}_S R_\lambda(Q_{\tilde{\theta}}(P_S^0))+ (1+\delta)\mathrm{E}_S \bigg[ \frac{16}{(n^1)^{1/p}}\bigg]\\
		& \times \log(1+\#\mathcal{Q}) \operatorname*{sup}_{Q\in \mathcal{Q}} \bigg[ \frac{C}{(n^1)^{1-1/p}} + \bigg(1.5C\times R_\lambda(Q)^{p-1}\bigg)^{1/p} \bigg(\frac{1+\delta}{\delta}\bigg)^{2/p-1}\bigg]
		\end{align*}
		Furthermore, by $R_\lambda(Q) \le C$ and  $\#\mathcal{Q} \le n^q $, 
		\begin{align*}
		L.H.S \le & (1+2\delta)\mathrm{E}_S R_\lambda(Q_{\tilde{\theta}}(P_S^0))+ (1+\delta)\mathrm{E}_S \bigg[ \frac{16}{(n^1)^{1/p}}\bigg]\\
		& \times q \text{log}(n) \times \bigg[ \frac{C}{(n^1)^{1-1/p}} + \bigg(1.5C^p\bigg)^{1/p} \bigg(\frac{1+\delta}{\delta}\bigg)^{2/p-1}\bigg]\\
		\le & (1+2\delta)\mathrm{E}_S R_\lambda(Q_{\tilde{\theta}}(P_S^0))+ (1+\delta)\mathrm{E}_S \bigg[ \frac{16}{(n^1)^{1/p}}\bigg]\\
		& \times q \text{log}(n) \times \bigg[ \frac{C}{(n^1)^{1-1/p}} + 1.5C \bigg(\frac{1+\delta}{\delta}\bigg)^{2/p-1}\bigg]
		\end{align*}
		
		Sample size of the test set is approximately a fixed proportion of the entire data, so $n^1 = O(n)$. Let $p=2$ and $\delta = 1/\sqrt{n}$, then the above inequality becomes 
		\begin{align*}
		\mathrm{E}_S R_\lambda(Q_{\hat{\theta}}(P_S^0))\le & \mathrm{E}_S R_\lambda(Q_{\tilde{\theta}}(P_S^0)) + 2\delta + (1+\delta)\mathrm{E}_S \bigg[ \frac{16}{(n^1)^{1/2}}\bigg]\\
		& \times q \text{log}(n) \times \bigg[ \frac{C}{(n^1)^{1-1/2}} + 1.5C \bigg]\\
		= & \mathrm{E}_S R_\lambda(Q_{\tilde{\theta}}(P_S^0)) + O\bigg(\frac{\text{log}(n)}{\sqrt{n}}\bigg).
		\end{align*}
	\end{proof}
	%%%%%%%%%%%%%%%%%%%%%%%%%%%%%%%%%%
	$\frac{\text{log}(n)}{\sqrt{n}}$ asymptotically goes to zero. As long as 
	\begin{equation} \label{convergence}
	\frac{\text{log}(n)/\sqrt{n}}{\mathrm{E}_S R_\lambda(Q_{\tilde{\theta}}(P_S^0)) } \rightarrow 0 \quad as \quad n\rightarrow \infty
	\end{equation}
	then $Q_{\theta_n}$ is asymptotically equivalent to the oracle estimator $Q_{\tilde{\theta}}$ in terms of their true risks averaged over data splittings when fitting the estimators on the training set,
	$$\frac{\mathrm{E}_S R_\lambda(Q_{\hat{\theta}}(P_S^0))}{\mathrm{E}_S R_\lambda(Q_{\tilde{\theta}}(P_S^0)) } \rightarrow 1 \quad as \quad n\rightarrow \infty.$$
	When equation \eqref{convergence} does not hold, then $Q_{\hat{\theta}}(P_S^0)$ achieves the $\frac{\text{log}(n)}{\sqrt{n}}$ rate:
	$$\mathrm{E}_S R_\lambda(Q_{\hat{\theta}}(P_S^0)) = O\bigg(\frac{\text{log}(n)}{\sqrt{n}}\bigg).$$
	
	This section shows that the performance of SL classification rule using joint thresholding in Section 4 is asymptotically equivalent to the oracle under some conditions. 
	
	\bibliographystyle{apacite}
	\bibliography{references1}

	\newpage
	\section*{Sample Codes}
	\begin{verbatim}
	### FUNCTIONS
	
	# Install functions for threshold line search under the 
	# weighted misclassification loss from github
	install_github("yizhenxu/TVLT")
	library(TVLT)
	
	# SL with two-step minimization
	SL.twostep = function(Y.train, fit.data.SLL, lambda){
	alpha = fit.data.SLL$coef
	Psi = fit.data.SLL$Z %*% fit.data.SLL$coef
	opt = Opt.nonpar.rule(Y.train, Psi, phi=0, lambda)
	cutoff = as.numeric(opt)[1]
	return(list(alpha = alpha, cutoff = cutoff))
	}
	
	# SL with CRS minimization
	SL.CRS = function(Y.train, fit.data.SLL, lambda){
	
	X = fit.data.SLL$Z # cross-validated library predictions  
	mod1 = nnls(as.matrix(X), Y.train) #non-negative least squares regression 
	initCoef = coef(mod1)
	initCoef[is.na(initCoef)] = 0
	
	ord = order(initCoef,decreasing=TRUE) # order the columns by initCoef
	X.SL = X[,ord]
	initCoef = initCoef[ord]
	b1 = initCoef/initCoef[1] #initial b
	G1 = X.SL%*%b1
	cR1 = as.numeric(Opt.nonpar.rule(Y.train, G1, phi=0, lambda)[1]) #initial c
	r = range(G1)
	
	Rtest <- function(t, lambda, X, Y) { #t[1] is cutoff, t[-1] is alpha
	t = matrix(unlist(t),ncol=1) 
	G =  X%*%t[-1] 
	result = sum(lambda*(G<=t[1])*Y+(1-lambda)*(G>t[1])*(1-Y))
	return(result)
	}
	# objective function
	fn = function(t) Rtest(t, lambda, X.SL, Y.train) # prepare for CRS (c,alpha)
	# starting values
	x0=c(cR1,b1) 
	# search region
	low = c(r[1]-0.5,rep(0,length(b1)))
	upp = c(r[2]+0.5,rep(5,length(b1)))
	# CRS
	crssol = crs2lm(x0 , fn , lower=low, upper=upp,
	maxeval =  10000, pop.size = 100000*(length(x0)+1), ranseed = seed,
	xtol_rel = 1e-6, nl.info = FALSE)
	bcrs = crssol$par[-1]
	norm = sum(bcrs)
	#normalize to make coefficients sum up to one
	bcrs = bcrs/norm
	cutoff = crssol$par[1]/norm
	#rewind to original order
	alpha = bcrs[ord]
	return(list(alpha = alpha, cutoff = cutoff))
	}
	
	### ANALYSIS
	
	SL.library = c("SL.randomForest","SL.glmnet","SL.leekasso")
	fit.data.SLL = SuperLearner(Y=Y.train, X=W.train,newX=W.test, SL.library = SL.library, 
	family = binomial(), method = "method.NNLS",verbose = FALSE)
	
	lambda = 0.7 # user-specified penalty for false-negatives in the weighted loss
	
	# two-step minimization
	mod1 = SL.twostep(Y.train, fit.data.SLL, lambda)
	
	# CRS minimization
	mod2 = SL.CRS(Y.train, fit.data.SLL, lambda)
	\end{verbatim}
	
\end{document}